\documentclass[letterpaper]{article} 
\usepackage[submission]{aaai23}  
\usepackage{times}  
\usepackage{helvet}  
\usepackage{courier}  
\usepackage[hyphens]{url}  
\usepackage{graphicx} 
\usepackage{natbib}  
\usepackage{caption} 
\usepackage{algorithm}
\usepackage{algorithmic}

\usepackage{newfloat}
\usepackage{listings}
\DeclareCaptionStyle{ruled}{labelfont=normalfont,labelsep=colon,strut=off} 
\floatstyle{ruled}
\newfloat{listing}{tb}{lst}{}
\floatname{listing}{Listing}
\usepackage{graphicx, mathtools}
\usepackage{amsthm, amssymb,enumitem,kantlipsum}
\usepackage{amsfonts}
\usepackage{amssymb}
\usepackage{xcolor}
\usepackage{subcaption}
\usepackage{algorithm}
\usepackage{dsfont}
\newtheorem{theorem}{Theorem}

\usepackage{soul}
\usepackage{multicol}
\usepackage{tabularx}
\usepackage[T1]{fontenc}    
\usepackage{hyperref}       
\usepackage{url}            
\usepackage{booktabs}       
\usepackage{nicefrac}       
\usepackage{microtype}      
\usepackage{xcolor}         

\usepackage{graphicx}
\usepackage{amsmath}
\usepackage{amsthm}
\usepackage{amssymb}
\usepackage{booktabs}
\usepackage{graphicx,subcaption}
\usepackage{enumitem,kantlipsum}

\title{On the Robustness of AlphaFold: A COVID-19 Case Study}

\author {
    Ismail R. Alkhouri\textsuperscript{\rm 1},
    Sumit Jha\textsuperscript{\rm 2},
    Andre Beckus\textsuperscript{\rm 3},
    George Atia\textsuperscript{\rm 1},
    Rickard Ewetz\textsuperscript{\rm 1},\\
    Arvind Ramanathan\textsuperscript{\rm 4},
    Susmit Jha\textsuperscript{\rm 5},
    Alvaro Velasquez\textsuperscript{\rm 6}
}
\affiliations {
    \textsuperscript{\rm 1} Electrical \& Computer Engineering Department, University of Central Florida, Orlando, FL 32816\\
    \textsuperscript{\rm 2} Computer Science Department, University of Texas at San Antonio, TX 78249\\
    \textsuperscript{\rm 3} Information Directorate, Air Force Research Laboratory, Rome, NY 13441\\
    \textsuperscript{\rm 4} Data Science and Learning, Argonne National Laboratory, Lemont, IL, 60439\\
    \textsuperscript{\rm 5} Computer Science Laboratory, SRI International, Menlo Park, CA, 94709\\
    \textsuperscript{\rm 6} Information Innovation Office, Defense Advanced Research Projects Agency, Arlington, VA 22203\\
}

\usepackage{bibentry}

\begin{document}

\maketitle

\begin{abstract}
Protein folding neural networks (PFNNs) such as AlphaFold predict remarkably accurate structures of proteins compared to other approaches. However, the robustness of such networks has heretofore not been explored. This is particularly relevant given the broad social implications of such technologies and the fact that biologically small perturbations in the protein sequence do not generally lead to drastic changes in the protein structure. In this paper, we demonstrate that AlphaFold does not exhibit such robustness despite its high accuracy. This raises the challenge of detecting and quantifying the extent to which these predicted protein structures can be trusted. To measure the robustness of the predicted structures, we utilize ($i$) the root-mean-square deviation (RMSD) and ($ii$) the Global Distance Test (GDT) similarity measure between the predicted structure of the original sequence and the structure of its adversarially perturbed version. We prove that the problem of minimally perturbing protein sequences to fool protein folding neural networks is \textbf{NP-complete}. Based on the well-established BLOSUM62 sequence alignment scoring matrix, we generate adversarial protein sequences and show that the RMSD between the predicted protein structure and the structure of the original sequence are very large when the adversarial changes are bounded by ($i$) 20 units in the BLOSUM62 distance, and ($ii$) five residues (out of hundreds or thousands of residues) in the given protein sequence. In our experimental evaluation, we consider 111 COVID-19 proteins in the Universal Protein resource (UniProt), a central resource for protein data managed by the European Bioinformatics Institute, Swiss Institute of Bioinformatics, and the US Protein Information Resource. These result in an overall GDT similarity test score average of around 34\%, demonstrating a substantial drop in the performance of AlphaFold.  
\end{abstract}

\section{Introduction}

Proteins form the building blocks of life as they enable a variety of vital functions essential to life and reproduction. Naturally occurring proteins are bio-polymers \textcolor{black}{typically} composed of 20 amino acids and this primary sequence of amino acids is well known for many proteins, thanks to high-throughput sequencing techniques. However, in order to understand the functions of different protein molecules and complexes, it is essential to comprehend their three-dimensional (3D) structures. 
Until recently, one of the grand challenges in structural biology has been the accurate determination of the 3D structure of the protein from its primary sequence. Such accurate predictive protein folding promises to have a profound impact on the design of therapeutics for diseases and drug discovery \cite{chan2019advancing}. 

%
\begin{figure}[t]
\centering
\includegraphics[width=8cm]{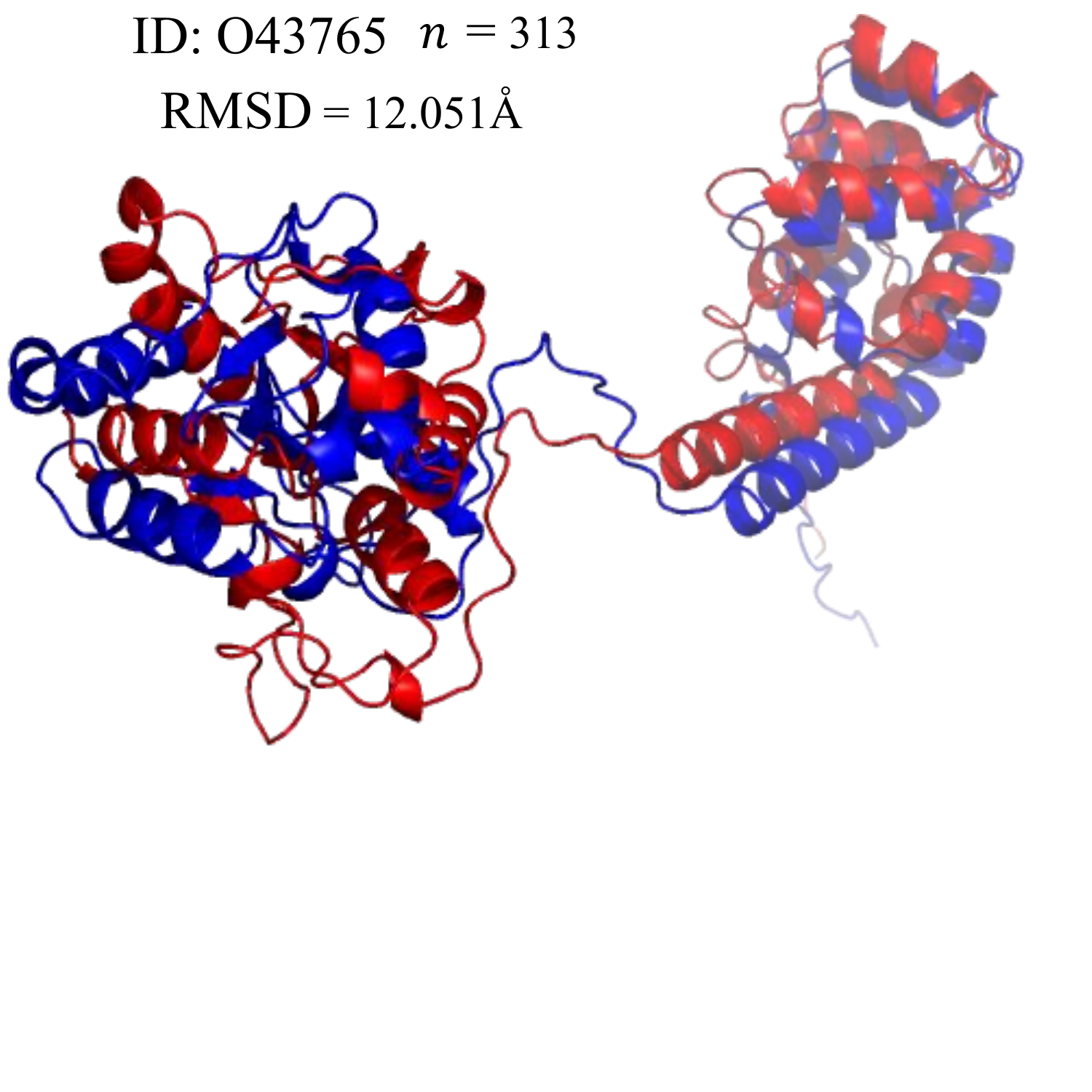}
\vspace{-2.8cm}
\caption{\small{The structure of the original (black) and adversarial (red) sequences predicted using AlphaFold for the Small glutamine-rich tetratricopeptide repeat-containing protein alpha sequence. The length of the protein sequence is denoted by $n$. For structures, after their alignment using PyMol \cite{pymol}, the Root Mean Square Deviation (RMSD) is given in Angstroms (equal to $10^{-10}$ meters and denoted by \r{A}).}}
\label{fig: intro}
\end{figure}

AlphaFold \cite{AlphaFold2021} achieved unparalleled success in predicting protein structures using neural networks and remains first at the Critical Assessment of protein Structure Prediction (CASP14), which corresponds to year 2020, competition. While this has been touted as a breakthrough for structural biology \cite{bagdonas2021case}, the robustness of its predictions has not yet been explored. The main contribution of this paper is to \textcolor{black}{demonstrate} the susceptibility of AlphaFold to adversarial sequences by generating several examples where protein sequences that vary only in five residues out of hundreds or thousands of residues result in very different 3D protein structures. \textcolor{black}{We present the problem of adversarial attacks on Protein Folding Neural Network (PFNN) 
and prove that the problem is \textbf{NP-complete}.} We use sequence alignment scores~\cite{henikoff1992amino} such as those derived from Block Substitution Matrices (BLOSUM62) to identify a space of similar protein sequences used in constructing adversarial perturbations. For the output structures, we leverage the standard metrics commonly used in CASP, namely ($i$) the root-mean-square deviation (RMSD) and ($ii$) the Global Distance Test (GDT) similarity measure between the predicted structure and the structure of its adversarially perturbed sequence. See Figure~\ref{fig: intro} and its caption for an example.

\textcolor{black}{Moreover, we conduct two experiments investigating the choice of the BLOSUM threshold and the use of the prediction, per-residue, confidence information obtained from AlphaFold}. Our experiments show that different input protein sequences have very different adversarial robustness as determined by the RMSD (GDT-TS) in the protein structure predicted by AlphaFold. These values range from \textcolor{black}{1.011\r{A} (0.43\%) to 49.531\r{A} (98.8\%)} when the BLOSUM62 distance between the original and adversarial sequences is bounded by a threshold of 20 units with a hamming distance of 5 residues only. Hence, our proposed approach is a first step in the direction of identifying protein sequences on which the predicted 3D structure cannot be trusted.


\section{Summary and Related work}

PFNNs~\cite{jumper2021highly,baek2021accurate} should be expected to obey the natural observation that biologically small changes in the sequence of a protein usually do not lead to drastic changes in the protein structure. Almost four decades ago, it was noted that two structures with  50\% sequence identity align within approximately 1\r{A} RMSD from each other~\cite{chothia1986relation}. Two proteins with even 40\% sequence identity and at least 35 aligned residues align within approximately 2.5\r{A}~\cite{sander1991database}. The phenomenon of sequence-similar proteins producing similar structures have also been observed in larger studies~\cite{rost1999twilight}. As with almost any rule in biology, a small number of counterexamples to the conventional wisdom of similar sequences leading to similar structures do exist, wherein even small perturbations can potentially alter the entire fold of a protein. However, such exceptions are not frequent and often lead to exciting investigations~\cite{Cordes_2000,tuinstra2008interconversion}.

\textcolor{black}{Manipulating the multiple sequence alignment step of AlphaFold has been studied in \cite{stein2021modeling} using in silico mutagenesis. However, there, the goal is not to study the robustness of the protein folding neural networks, but rather to enhance the prediction capability of AlphaFold in terms of the intrinsic conformational heterogeneity of proteins. The authors in \cite{del2022sampling}, present a method that manipulates inputs to obtain diverse distinct structures that are absent from the AlphaFold training data. Using membrane proteins, the authors show that their method enhances the multiple sequence alignment step while generating more accurate structures.}

The work in \cite{jha2021protein} is aimed at generating adversarial sequences in order to cause significant damage to the output predicted structure of RosettaFold \cite{baek2021accurate}, which, according to CASP, is the second best protein folding neural network. However, the authors only show results for a few proteins and do not consider all the standard metrics for measuring the output structures. In contrast, in this paper, we present results for more than 100 sequences, derive a complexity proof for the problem of finding adversarial protein sequences, and, based on the CASP competition, utilize all the standard metrics for measuring the output structures.


%
%
%
%

\subsection{Robustness Metric using Adversarial Attacks}


The similar-sequence implies similar-structure paradigm dictates that PFNNs should make robust predictions. 
Given a protein sequence of $n$ residues $S=s_1 s_2 \dots s_n$ with a three-dimensional structure $\mathcal{A}(S) = (x_1, y_1, z_1), \dots, (x_n,y_n,z_n)$, we define a notion of biologically similar sequences $\mathcal{V}$ using Block Substitution Matrices (BLOSUM)~\cite{henikoff1992amino}, and then employ formulations of adversarial attacks~\cite{goodfellow2018making} on PFNNs within this space of similar sequences to identify a sequence $S_{\textrm{adv}} \in \mathcal{V}$ that produces a maximally different three-dimensional structure $\mathcal{A}(S_{\textrm{adv}})$. We then compute the RMSD and GDT between the structures for the original and adversarial inputs ($\mathcal{A}(S)$ and $\mathcal{A}(S_{\textrm{adv}})$), and use these metrics as the robustness measure. If the RMSD (GDT) is small (high), the response of the PFNN is deemed robust; a large (small) RMSD (GDT) indicates that the predicted structure is not robust. 

\subsection{BLOSUM Similarity Measures}

Given two sequences of $n$ residues $S=s_1 s_2 \dots s_n$ and $S'=s'_1 s'_2 \dots s'_n$, in which every residue $s_i$ (or $s'_i$) is from the set $\mathcal{X} = \{A,R,N,D,C,Q,E,G,H,I,L,K,M,F,P,\\S,T,W,Y,V\}$ of amino acids, a natural question is how to compute the sequence similarity $D_{\textrm{seq}}$ between these proteins. A naive approach would be to count the number of residues that are different, i.e., the Hamming distance. However, an analysis of naturally occurring proteins shows that not all changes in residues have the same impact on protein structures. Changes to one type of residue are more likely to cause structural variations than changes to another type of residue.

Early work in bioinformatics focused on properties of amino acids and reliance on genetic codes. However, more modern methods have relied on the creation of amino acid scoring matrices that are derived from empirical observations of frequencies of amino acid replacements in homologous sequences~\cite{dayhoff197822,jones1992rapid}. The original scoring matrix, called the PAM250 matrix, was based on empirical analysis of $1572$ mutations observed in $71$ families of closely-related proteins that are 85\% or more identical after they have been aligned. The PAM1 model-based scoring matrix was obtained by normalizing the frequency of mutations to achieve a 99\% identity between homologous proteins. These results were then extrapolated to create the PAM10, PAM30, PAM70 and PAM120 matrices with 90\%, 75\%, 55\%, and 37\% identity between homologous proteins.

Another interesting approach~\cite{henikoff1992amino} to understanding protein similarity is the direct counting of replacement frequencies using the so-called Block Substitution Matrices (BLOSUM). Instead of relying solely on sequences of homologous proteins that are relatively harder to find, the BLOSUM approach focuses on identifying conserved blocks or conserved sub-sequences in a larger variety of proteins potentially unrelated by evolutionary pathways and counts the frequency of replacements within these conserved sub-sequences. BLOSUM62 (\textcolor{black}{Figure~\ref{fig: blosum matrix}}), BLOSUM80 and BLOSUM90 denote block substitution matrices that are obtained from blocks or subsequences with at least 62\%, 80\%, and 90\% similarity, respectively. The BLOSUM matrix $[B_{ij}]$ is a matrix of integers where each entry denotes the similarity between residue of type $b_i \in \mathcal{X}$ and type $b_j \in \mathcal{X}$.

We identify the space of biologically similar sequences $\mathcal{V}$ for a given protein sequence $S$ with respect to the BLOSUM distance. 
We expect the predicted structures for the similar sequences to be similar. If there is a large RMSD (or small GDT) between
the predicted structure $\mathcal{A}(S)$ and the structure of the adversarial sequence $\mathcal{A}(S_{\textrm{adv}})$, it would reflect a lack of robustness in the prediction of the network. We adopt a sequence similarity measure that counts replacement frequencies in conserved blocks across different proteins.

\section{Approach}
\label{sec:approach}

Our approach to evaluating the robustness of PFNNs is based on two main ideas: ($i$) the existence of adversarial examples in PFNNs that produce adversarial structures possibly very different from the original structure, and ($ii$) the use of BLOSUM matrices for identifying a neighborhood of a given sequence that are biologically similar and hence expected to have similar 3D structures. We utilize the RMSD and GDT between the structure of an original protein sequence and the structure of the adversarial sequence as a measure of robustness of a protein folding network on the given input. In this work, we focus on the state-of-the-art AlphaFold model\textcolor{black}{, the winner of the 1st place in CASP2020.}
\subsection{Sequence  Similarity Measures}

Given two sequences $S=s_1 s_2 \dots s_n$ and $S'=s'_1 s'_2 \dots s'_n$,  the BLOSUM distance between the two sequences is given by Equation (\ref{eqn: seq dis}) below. For an illustrative example of $D_{\mathrm{seq}}$, see Figure~\ref{fig: blosum ex}. 
\begin{equation} \label{eqn: seq dis}
D_{\mathrm{seq}}(S,S') = \sum_{i\in [n]} \left( B_{{s_i}{s_i}} - B_{{s_i}{s'_i}} \right)\:.
\end{equation}
%
%
\begin{figure}[h]
\centering
\includegraphics[width=9cm]{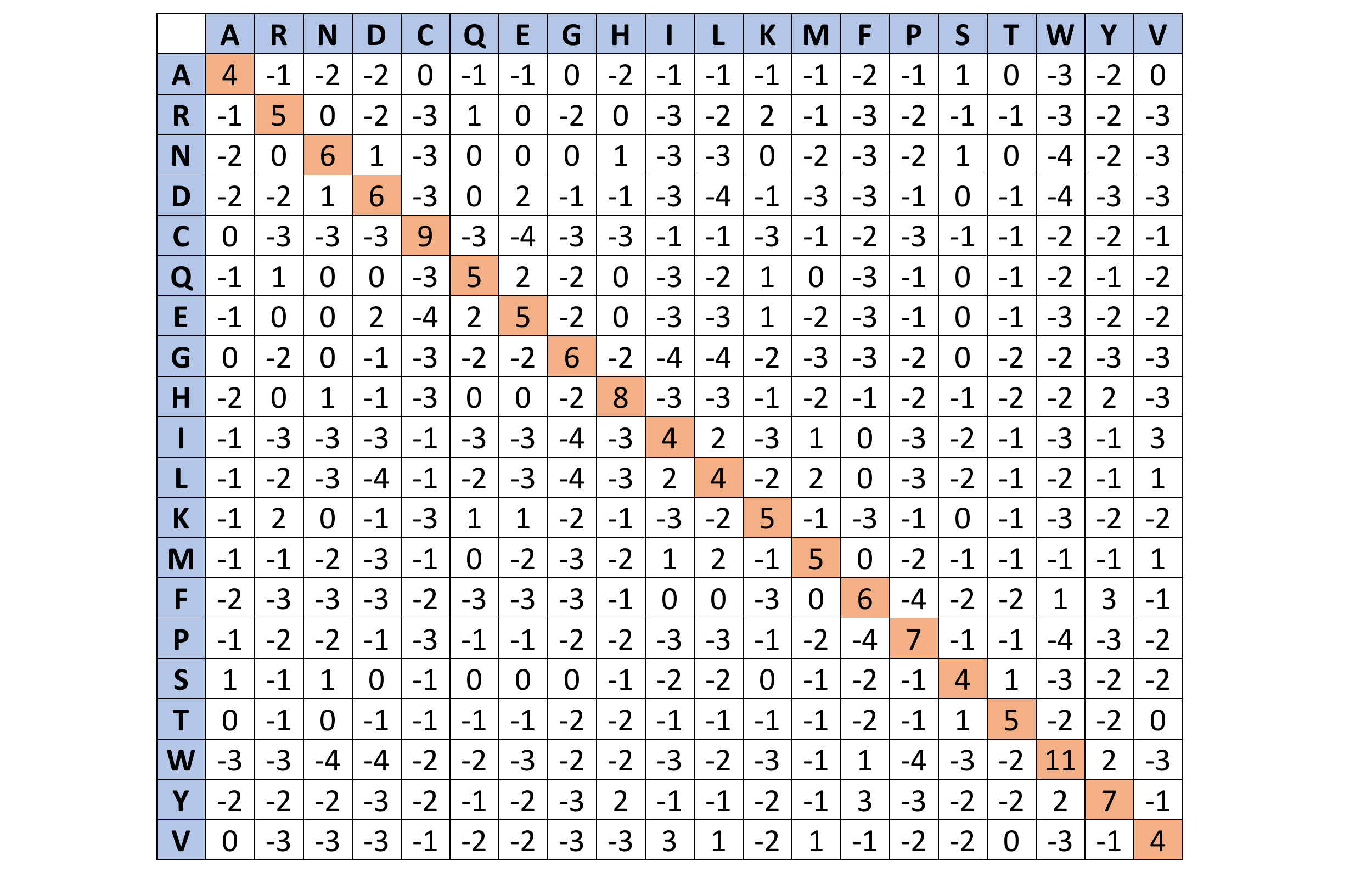}
\vspace{-0.6cm}
\caption{\small{The BLOSUM62 matrix. }} 
\label{fig: blosum matrix}
\end{figure}
\begin{figure*}[h]
\centering
\includegraphics[width=16cm]{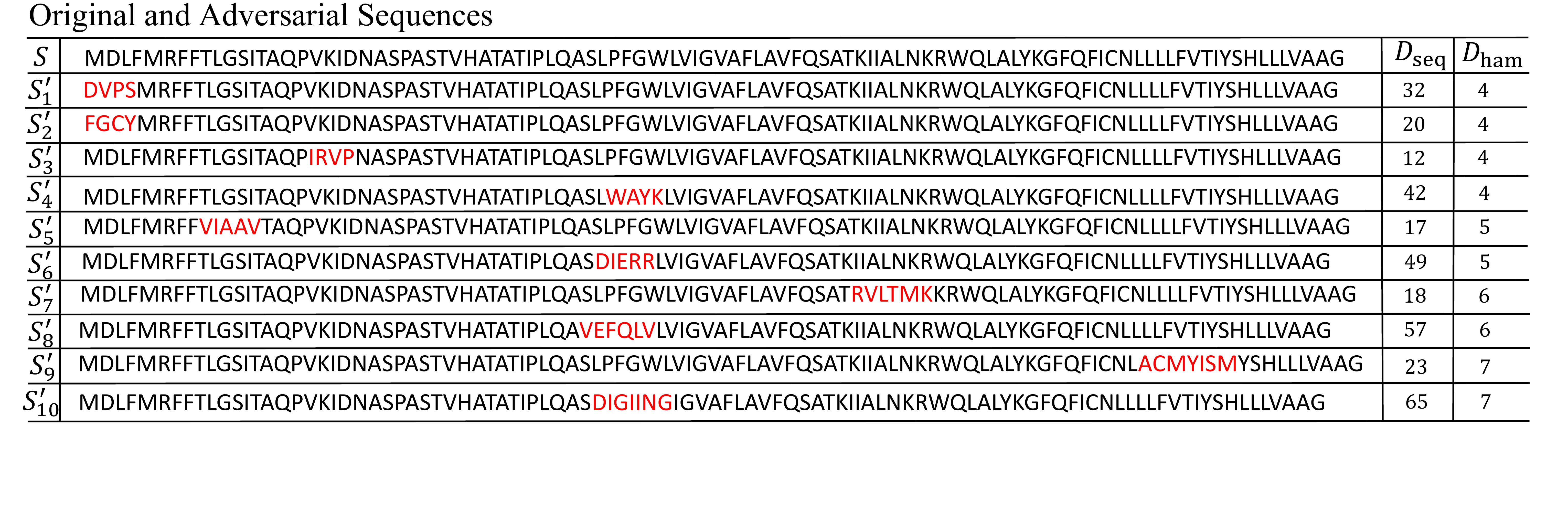}
\vspace{-1.11cm}
\caption{\small{The original sequence $S$ is followed by 10 sequences generated by changing 4, 5, 6, and 7 residues. The sequences are samples from the space in \eqref{eqn: space} with different values of $L$ and $H$. The distance $D_{\mathrm{seq}}$ is calculated using \eqref{eqn: seq dis}.}} 
\label{fig: blosum ex}
\end{figure*}
%


\subsection{Output Structural Measure}
Given a sequence of $n$ residues $S=s_1 s_2 \dots s_n$, its three dimensional structure $\mathcal{A}(S)$ is an ordered $n$-tuple of three-dimensional co-ordinates $(x_1, y_1, z_1), \dots (x_n,y_n,z_n)$. Our goal is to utilize a structural distance measure that captures the variations in the two structures $\mathcal{A}(S)$ and $\mathcal{A}(S')$ and is invariant to rigid-body motion. Therefore, in this work, we use standard structural distances, namely the RMSD, measured in \r{A}, and the GDT with its two variants: ($i$) the Total Score (TS) and ($ii$) the High Accuracy (HA) \cite{zemla2003lga}.

Given the output structure of the adversarial sequence $\mathcal{A}(S')$, an alignment algorithm is employed before computing the RMSD and GDT measures between the two structures of interest. We use the alignment procedure implemented in PyMOL \cite{pymol} to align $\mathcal{A}(S')$ with regard to the target structure $\mathcal{A}(S)$. Let the aligned structure be denoted by $\hat{\mathcal{A}}(S') = (\hat{x}'_1, \hat{y}'_1, \hat{z}'_1), \dots, (\hat{x}'_n, \hat{y}'_n, \hat{z}'_n)$. Then, the RMSD, measured in \r{A}, is obtained as 
\begin{equation} \label{eqn: RMSD}
\begin{gathered}
\textrm{RMSD}(\mathcal{A}(S), \hat{\mathcal{A}}(S')) = \sqrt{ \frac{1}{n} \sum_{i\in [n]} d(\mathcal{A}(S)_i, \hat{\mathcal{A}}(S')_i) } \:,
\end{gathered}
\end{equation}
where $d(\mathcal{A}(S)_i, \hat{\mathcal{A}}(S')_i) =    (x_i-\hat{x}'_i)^2 + (y_i-\hat{y}'_i)^2 + (z_i-\hat{z}'_i)^2 \:$ and $\mathcal{A}(S)_i$ represents the 3D carbon-alpha coordinates of the $i^{\textrm{th}}$ residue. Using the carbon-alpha coordinates is the standard approach in CASP \cite{zemla2003lga}. \par

Another standard metric for gauging the similarity of protein structures is the GDT similarity measure, introduced by \cite{zemla2003lga} and commonly used in the CASP competition along with the RMSD. In some cases, the latter is known to be sensitive to outliers \cite{zemla2003lga}. The GDT score returns a value in $[0,1]$ where $1$ indicates identical structures, and is computed with respect to four thresholds, $\delta_j$, as
\begin{equation} \label{eqn: GDT}
\begin{gathered}
\textrm{GDT}(\mathcal{A}(S), \hat{\mathcal{A}}(S')) = \\
\frac{1}{4n} \sum_{j\in [4]} \sum_{i\in [n]} \textbf{1}\big( d(\mathcal{A}(S)_i, \hat{\mathcal{A}}(S')_i) < \delta_j \big) \:,
\end{gathered}
\end{equation}
where the thresholds $\delta_1$, $\delta_2$, $\delta_3$, and $\delta_4$ for TS (HA) are given by $1 (0.5)$, $2 (1)$, $4 (2)$, and $8 (4)$ for $j$ equals to $1$, $2$, $3$, and $4$ respectively, and $\textbf{1}(\cdot)$ is the indicator function. 
In \eqref{eqn: GDT}, each $j\in [4]$ reflects the number of residues in the structures for which the distance is \textcolor{black}{less than} $\delta_j$.

\subsection{Adversarial Attacks on PFNNs}

Small carefully crafted changes in a few pixels of input images cause well-trained neural networks with otherwise high accuracy to consistently produce incorrect responses in domains such as computer vision~\cite{croce2020robustbench,andriushchenko2020square,bai2020improving,croce2021mind}. 
Given a neural network $\mathcal{A}$ mapping a sequence $S$ of residues to a three-dimensional geometry $\mathcal{A}(S)$ describing the structure of the protein, we seek to obtain a sequence $S'$ such that the sequence similarity measure $D_{\mathrm{seq}}(S, S')$ between $S$ and $S'$ is small and some structural distance measure $D_{\mathrm{str}}(\mathcal{A}(S), \mathcal{A}(S'))$ is maximized. This can be achieved by solving the following optimization problem
%
%
\begin{equation} \label{eqn: opt}
\begin{gathered}
\max_{S'} D_{\mathrm{str}}\left(\mathcal{A}(S), \mathcal{A}(S')\right)  \; \text{s.t.} \; D_{\mathrm{seq}}(S,S') \leq L\:.
\end{gathered}
\end{equation}

\textcolor{black}{In our experiments, we set $L=20$ and $D_{\mathrm{str}}$ as the RMSD measure. Given the discrete nature of the input sequences, well-known methods for generating adversarial examples (e.g. gradient-based methods) fail to produce valid and accurate results. 
As such, we propose a solution based on a brute-force exploration in the space of biologically similar sequences that, given a sequence of interest $S$ with $n$ residues, can be defined as} 
\begin{equation} \label{eqn: space}
\begin{gathered}
\mathcal{V}_{L,H}(S) = \{S'\in \mathcal{X}^{n} \mid D_{\mathrm{seq}}(S,S')\leq L \text{ and } \\ D_{\mathrm{ham}}(S,S')\leq H \}\:,
\end{gathered}
\end{equation}
where $\mathcal{X}^n$ is the set of all possible sequences over $\cal X$ of length $n$, $D_{\mathrm{ham}}$ is the hamming distance, and $H$ is a predefined threshold. For long sequences, the search space can be extensively large. Therefore, we select random samples from $\mathcal{V}_{L,H}(S)$ and choose the sequence that returns the maximum value based on the RMSD measure. Our approach to generating adversarial sequences falls under the class of black-box attacks. This means that we only have access to the output of the network \cite{papernot2017practical}. 

It is worth noting that the inference time of complex protein folding systems, which apply multiple processing and alignment steps prior to the use of any neural network, such as AlphaFold is extremely high compared to NN-based image classifiers. The forward pass of such systems involves a large number of computations. This fact, along with the discrete nature of the input space, are the bottleneck of developing more complex black-box attacks \cite{mahmood2021back}, which in general require a high number of queries. 


\section{Complexity}

In this section, we formalize the problem of generating an adversarial attack for PFNNs and establish its complexity.

\newtheorem{definition}{Definition}
\begin{definition}[PFNN Adversarial Attack (PAA) Problem]
Given a learning model $\mathcal{A}(.\:; \theta): \mathcal{X}^n \rightarrow (\mathbb{R} \times \mathbb{R} \times \mathbb{R})^n$ mapping residues to 3-dimensional coordinates and parameterized by $\theta$, a sequence $S \in \mathcal{X}^n$, and a sequence alignment scoring matrix $B$, find an input sequence $S' \in \mathcal{X}^n$ such that $D_\mathrm{seq}(S, S') \leq L$ and $D_\mathrm{str}(\mathcal{A}(S), \mathcal{A}(S')) \geq U$, where the bounds $L$ and $U$ and distance functions $d$ and $D$ are given.
\label{def:PAA}
\end{definition}

We prove that the PAA problem is \textbf{NP-complete}. This establishes that, in general, there is no polynomial-time solution to the PAA problem unless \textbf{P} = \textbf{NP}. Due to this complexity and for ease of presentation, we adopt simple perturbation attacks for our experiments in the next section. We begin by defining the \textbf{NP-complete} problem to be reduced to an instance of the PAA problem.

\begin{definition}[CLIQUE Problem]
Given an undirected graph $G = (V, E)$ and an integer $k$, find a fully connected sub-graph induced by $V' \subseteq V$ such that $|V'| = k$.
\label{def:clique}
\end{definition}

\begin{theorem}
The PFNN Adversarial Attack (PAA) problem in Definition \ref{def:PAA} is \textbf{NP-complete}.
\label{theorem:complexity}
\end{theorem}
\begin{proof}
It is easy to verify that the PAA problem is in \textbf{NP} since, given a solution sequence $S'$, one can check whether the constraints $D_{\mathrm{seq}}(S,S') \leq L$ and $D_{\mathrm{str}}(\mathcal{A}(S), \mathcal{A}(S')) \geq U$ are satisfied in polynomial time. It remains to be shown whether the PAA problem is \textbf{NP-hard}. We establish this result via a reduction from the CLIQUE problem in Definition \ref{def:clique}. Given a CLIQUE instance $\left< G = (V, E), k \right>$ with $|V| = n$ and $|E| = m$, we construct its corresponding PAA instance $\left< \mathcal{A}(.\:; \theta), S, B, L, U \right>$ as follows. Without loss of generality, let us consider a restricted version of the PAA problem where there are only two residue types $\{N, K\}$ with the corresponding BLOSUM62 sub-matrix $B' = 6 \cdot I$, where $I$ denotes the identity matrix. Following the one-hot representation of residues adopted in \cite{jumper2021highly}, any input tensor over $\{N, K\}$ is represented as a one-hot encoding $S^\text{in} \in (\mathbb{B} \times \mathbb{B})^n$ to be used as an input tensor to $\cal{A}$, where $s^\text{in}_{i0} = 1$ ($s^\text{in}_{i1} = 1$) denotes that residue $s^\text{in}_i$ is of type $N$ ($K$). Let $S = (N, N, \dots, N)$ denote the all-$N$ sequence. We set $L = 6k$ and $U = \frac{k(k-1)}{2} \sqrt{\frac{3}{n}}$. The connectivity structure of $\cal{A}$ is derived from the edges $E$ in the CLIQUE instance as follows. The first column of the input tensor corresponding to $s^\text{in}_{i0}$ for all $i \leq n$ is disconnected from the network and the second column corresponding to $s^\text{in}_{i1}$ is connected to $\cal{A}$ such that, for each edge $(v_i, v_j) \in E$, we have a connection from $s^\text{in}_{i1}$ and $s^\text{in}_{j1}$ to each of the three outputs in the first three-dimensional coordinate of $\mathcal{A}(S^\text{in})_1$. All connections have a weight of unity and this defines the parameters $\theta$ of the model $\cal{A}$. Therefore, without loss of generality, we are only considering the first of the $n$ output three-dimensional coordinates $\mathcal{A}(S^\text{in})_1$. In particular, these values keep track of the number of edges induced by the vertices in $G$ corresponding to the non-zero entries in $s^\text{in}_{11}, \dots, s^\text{in}_{1n}$. We now prove that there is a clique of size $k$ in $G$ if and only if there is a feasible solution $S^\text{in} = S'$ to the reduced PAA instance.

($\implies$) Assume there is a clique of size $k$ in $G$. We can derive a feasible solution $S'$ to the reduced PAA instance as follows. For every vertex $v_i \in V$ (not) in the clique, let ($s'_{i0} = 1$) $s'_{i1} = 1$. Since $S$ is the all-$N$ sequence, its corresponding one-hot encoding consists of $s_{i0} = 1$ for all $1 \leq i \leq n$. Thus, the corresponding BLOSUM62 distance is

%
\begin{equation}
\begin{gathered}
D_{\mathrm{seq}}(S,S') = \sum_{1 \leq i \leq n} \left( 6 - 6 \cdot \mathbf{1}(s_i \neq s'_i) \right) = 6k\:.
\end{gathered} 
\end{equation}
%


This satisfies the sequence alignment constraint defined by $D_{\mathrm{seq}}(S,S') \leq L = 6k$. Furthermore, the solution $S'$ induces outputs of $x'_1 = y'_1 = z'_1 = k(k-1)/2$, leading to an RMSD of $U$. Without loss of generality, we omit the alignment step in computing the RMSD and therefore assume that $\mathcal{A}(S') = \hat{\mathcal{A}}(S')$.
The corresponding RMSD distance $D_{\mathrm{str}}( \mathcal{A}(S),\hat{\mathcal{A}}(S'))$ in output predictions is presented below. Recall that $x_1 = y_1 = z_1 = 0$ for the the all-$N$ sequence $S$ because its corresponding column in the one-hot encoding is disconnected from the network.
%
%

\begin{equation}
\begin{gathered}
D_{\mathrm{str}}( \mathcal{A}(S),\mathcal{A}(S')) = \sqrt{ \frac{1}{n} \sum_{i\in [n]} d(\mathcal{A}(S)_i, \hat{\mathcal{A}}(S')_i) } \\ 
= \sqrt{\frac{1}{n} \left[ 3 \left(0 - \frac{k (k - 1)}{2} \right)^2 \right]} = \frac{k(k-1)}{2} \sqrt{\frac{3}{n}} \:.
\end{gathered} 
\end{equation}


Thus, the constraint $D_{\mathrm{str}}(S, S') \geq U = \frac{k(k-1)}{2} \sqrt{\frac{3}{n}}$ is satisfied.

($\impliedby$) We prove the contrapositive. That is, if there is no clique of size $k$ in $G$, then the reduced PAA instance is infeasible. We proceed by showing that there must be exactly $k$ non-zero entries in the column vector $\{s'_{i1} | i \leq n\}$ in order to satisfy constraints $D_{\mathrm{seq}}(S,S') \leq L = 6k$ and $D_{\mathrm{str}}(\mathcal{A}(S),\mathcal{A}(S')) \geq U$ and that, if there is no clique of size $k$, then there is no choice of $k$ non-zero entries in $\{s'_{i1} | i \leq n\}$ that will satisfy these constraints. Let $k'$ denote the number of non-zero entries in $\{s'_{i1} | i \leq n\}$. To satisfy $D_{\mathrm{seq}}(S,S') \leq L = 6k$, it follows that $k' \leq k$. If $k' < k$, then the maximum value of $D_{\mathrm{str}}(\mathcal{A}(S), \mathcal{A}(S'))$ is $\frac{k'(k'-1)}{2} \sqrt{\frac{3}{n}} < \frac{k(k-1)}{2} \sqrt{\frac{3}{n}}$ and denotes to the case where the $k'$ non-zero entries correspond to a clique of size $k'$ in $G$. The strict inequality is due to the monotonically increasing nature of this equation. Therefore, it must be that $k = k'$ and we have outputs $x'_1 = y'_1 = z'_1 = k(k-1)/2$ as before. Suppose that the $k'$ non-zero entries in $\{s'_{i1} | i \leq n\}$ do not correspond to a clique in $G$. Then the values $x'_1$, $y'_1$, and $z'_1$ output by $\mathcal{A}$ and corresponding to the number of edges induced by the chosen non-zero entries would be strictly less than $k(k-1)/2$. Therefore, we would have $D_{\mathrm{str}}(\mathcal{A}(S),\mathcal{A}(S')) < U$. This proves that the reduced PAA is infeasible.
\end{proof}

\section{Experimental Results}
\begin{table}[t]
\caption{\small{RMSD results when $L\in \{20,30,40\}$.}}
\label{table:blosum}\centering
\vspace{-0.3cm}
 \scalebox{0.78}{\begin{tabular}{||c c | c c c c c c ||} 
 \hline
  Seq. ID  & $n$ & $L$ & RMSD & $\mu_{\textrm{all}}$ & $\mu_{\textrm{diff}}$ & $\mu'_{\textrm{all}}$  & $\mu'_{\textrm{diff}}$ \\ 
 \hline\hline

 Q14653  &  427  &  20 &  18.87  &  79.76 & 92.92 & 79.46 & 86.29 \\ 
\hline
Q14653  &  427  &  30 &  22.42  &  79.76 & 93.15  & 77.45 & 64.12 \\ 
\hline
Q14653  &  427  &  40 &  28.28  &  79.76 & 90.49 & 79.42 & 69.026 \\ 
\hline\hline

Q5BJD5  &  291  &  20 &  14.311 &  82.23 & 89.77  &  80.6  & 80.64 \\ 
\hline
Q5BJD5  &  291  &  30 &  15.708  &  82.23 & 59.26 &  83.13 & 43.53 \\ 
\hline
Q5BJD5  &  291  &  40 &  17.132  &  82.23 & 62.02 &  83.21 & 62.83 \\ 
\hline\hline

P59595  &  422  &  20 & 24.321 & 68.25  &  91.44 & 67.05  &  89.51   \\ 
\hline
P59595  &  422  &  30 & 30.139 & 68.25  &  93.142 & 67.44 &  89.29  \\ 
\hline
P59595  &  422  &  40 & 30.675 & 68.25  &  46.87 & 66.4   &  29.33  \\ 
\hline\hline


P0DTC9  &  419  &  20 &  26.51  &  68.39 & 80.32 &  68.09 & 80.316 \\ 
\hline
P0DTC9  &  419  &  30 &  26.27  &  68.39 & 68.05 &  68.61 & 65.18 \\ 
\hline
P0DTC9  &  419  &  40 &  31.33  &  68.39 & 40.52 &  67.76 & 35.56 \\ 
\hline\hline


P07711  &  333  &  20 &  7.09  &  93.68  & 92.4 & 93.2 & 81.12 \\ 
\hline
P07711  &  333  &  30 &  8.52  &  93.68  & 95.91 & 92.95 & 92.69 \\ 
\hline
P07711  &  333  &  40 &  9.246  &  93.68 & 95.91 & 92.85 & 95.76 \\ 
\hline\hline

Q9Y397  &  364  &  20 &  11.184  & 84.24   & 97.35 & 83.85 & 95.81 \\ 
\hline
Q9Y397  &  364  &  30 &  11.828  &  84.24  & 95.91 & 83.51 & 85.416 \\ 
\hline
Q9Y397  &  364  &  40 &  14.222  &  84.24  & 95.91 & 83.71 & 89.79 \\ 
\hline\hline

 \end{tabular}}
\vspace{ -0.3cm}
\end{table}

\begin{table}[t]
\caption{\small{RMSD results for the three considered categories.}}
\vspace{ -0.3cm}
\label{table:conf}\centering
 \scalebox{0.78}{\begin{tabular}{||c c | c c c c c c ||} 
 \hline
  Seq. ID  & $n$ & Category & RMSD & $\mu_{\textrm{all}}$ & $\mu_{\textrm{diff}}$ & $\mu'_{\textrm{all}}$  & $\mu'_{\textrm{diff}}$ \\ 
 \hline\hline

 Q01629  &  132  &  MIN. &  6.02  &  64.63 & 32.44 &  60.99 & 36.99 \\ 
\hline
Q01629  &  132  &  AVG. &  19.92  &  64.63 & 64.75 &  63.77 & 69.57 \\ 
\hline
Q01629  &  132  &  MAX. &  19.906  &  64.63 & 66.99 &  90.21 & 90.19 \\ 
\hline\hline

Q5BJD5  &  291  &  MIN. &  14.023 &  82.23 & 38.86 &  81.22 & 37.79 \\ 
\hline
Q5BJD5  &  291  &  AVG. &  14.232  &  82.23 & 82.24 &  81.17 & 77.23 \\ 
\hline
Q5BJD5  &  291  &  MAX. &  13.567  &  82.23 & 98.17 &  82.42 & 98.1 \\ 
\hline\hline

P59595  &  422  &  MIN. & 24.74 &  68.25  &  29.13 & 67.57 &  31.5  \\ 
\hline
P59595  &  422  &  AVG. & 28.164 & 68.25  &  69.44 & 68.69 &  69.04  \\ 
\hline
P59595  &  422  &  MAX. & 24.62 & 68.25  &  96.14 & 67.51 &  96.44  \\ 
\hline\hline

P59633  &  154  &  MIN. &  21.67  &  44.82 & 27.14 &  44.15 & 38.75 \\ 
\hline
P59633  &  154  &  AVG. &  21.52  &  44.82 & 45.1 &  43.8 & 42.26 \\ 
\hline
P59633  &  154  &  MAX. &  23.13  &  44.82 & 61.26 &  46.13 & 54.84 \\ 
\hline\hline

P0DTC9  &  419  &  MIN. &  25.593  &  68.39 & 28.46 &  67.9 & 28.83 \\ 
\hline
P0DTC9  &  419  &  AVG. &  21.767  &  68.39 & 68.37 &  68.5 & 70.83 \\ 
\hline
P0DTC9  &  419  &  MAX. &  23.685  &  68.39 & 97.1 &  68.64 & 96.94 \\ 
\hline

 \end{tabular}}
\vspace{ -0.5cm}
\end{table}

\begin{figure*}[t]
\centering
\includegraphics[width=18cm]{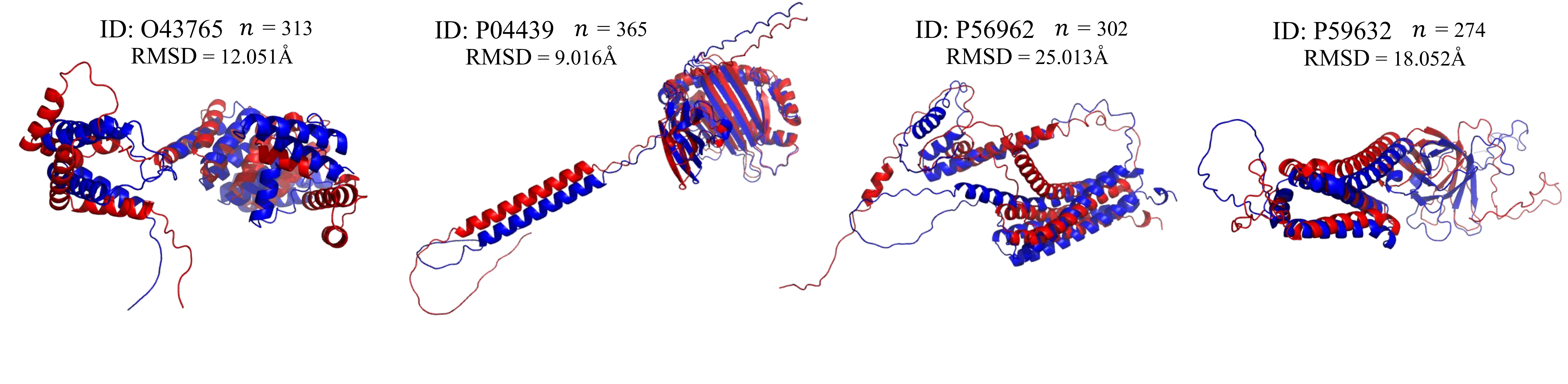}
\vspace{-1.55cm}
\caption{\small{The structures of the original (black) and adversarial (red) sequences from AlphaFold. The 3D plots, aligned using PyMol \cite{pymol}, are for proteins O43765 (first), P04439 (second), P56962 (third), and P59632 (fourth). For structure differences, the RMSD values are reported. \textcolor{black}{The structures of the complete list of sequences are given in the supplementary material.}}}
\vspace{-0.3cm}
\label{fig: output}
\end{figure*}

\begin{table*}[t]
\caption{\small{RMSD, GDT-TS, and GDT-HA results using the full database AlphaFold configuration with $L=20$ and $H=5$. The average columns correspond to 20 adversarial samples for each protein ID. The complete table is placed in the supplementary material.}}
\vspace{ -0.3cm}
\label{table:full db configurations}\centering
 \scalebox{0.78}{\begin{tabular}{||c c c | c c | c c | c c | c ||} 
 \hline
  Seq. ID  & $n$ & Similarity (\%) &  RMSD & Avg. RMSD  & GDT-TS (\%) & Avg. GDT-TS (\%) & GDT-HA (\%) & Avg. GDT-HA (\%) &  run-time (days)  \\
 \hline\hline

\hline
O43765  &  313  &  98.4026  &  14.438  &  9.1741  &  13.9776  &  35.4832  &  2.8754  &  17.6358  &  1.6068  \\ 
\hline
\hline
P56962  &  302  &  98.3444  &  22.301  &  15.8695  &  12.3344  &  18.6921  &  3.4768  &  5.803  &  0.5959  \\ 
\hline
\hline
P04439  &  365  &  98.6301  &  6.162  &  3.7942  &  47.7397  &  68.2705  &  25.0  &  45.774  &  0.6429  \\ 
\hline
\hline
Q99836  &  296  &  98.3108  &  8.761  &  5.2907  &  24.1554  &  46.6723  &  7.6858  &  26.2584  &  0.6246  \\ 
\hline
\hline
P59632  &  274  &  98.1752  &  13.018  &  8.4704  &  24.8175  &  41.0401  &  9.0328  &  21.6834  &  0.5214  \\ 
\hline
\end{tabular}}
\vspace{ -0.5cm}
\end{table*}

\begin{table*}[t]
\caption{\small{Prediction confidence results using the full database AlphaFold configuration with $L=20$.} }
\vspace{ -0.3cm}
\label{table:full db configuration confidence results}\centering
 \scalebox{0.78}{\begin{tabular}{||c c | c | c c | c c | c c | c c||} 
 \hline
  Seq. ID  & $n$ & RMSD &  $\mu_{\textrm{all}}$ & $\sigma_{\textrm{all}}$  & $\mu_{\textrm{diff}}$ & $\sigma_{\textrm{diff}}$ &  $\mu'_{\textrm{all}}$ &  $\sigma'_{\textrm{all}}$  &  $\mu'_{\textrm{diff}}$ &  $\sigma'_{\textrm{diff}}$  \\
 \hline\hline
\hline
O43765  &  313  &  14.438  &  80.221  &  19.634  &  94.786  &  1.027  &  80.554  &  19.423  &  93.71  &  1.392  \\ 
\hline
\hline
P56962  &  302  &  22.301  &  69.172  &  23.753  &  96.516  &  0.409  &  69.342  &  23.759  &  96.54  &  0.463  \\ 
\hline
\hline
P04439  &  365  &  6.162  &  86.845  &  18.995  &  44.23  &  2.704  &  86.921  &  19.068  &  44.678  &  3.968  \\ 
\hline
\hline
Q99836  &  296  &  8.761  &  81.213  &  13.817  &  78.914  &  7.454  &  80.918  &  13.971  &  72.198  &  6.253  \\ 
\hline
\hline
P59632  &  274  &  13.018  &  58.367  &  18.783  &  66.136  &  2.029  &  57.364  &  18.794  &  60.926  &  4.362  \\ 
\hline

\end{tabular}}
\vspace{ -0.3cm}
\end{table*}

\begin{table*}[t]
\caption{\small{Overall Prediction and attack results for the reduced and full database configurations of AlphaFold.}}
\vspace{ -0.3cm}
\label{table:overall results}\centering
 \scalebox{0.78}{\begin{tabular}{|| c | c c | c c | c c | c c | c c ||} 
 \hline
  Configuration.  & Avg. $n$ & Std. $n$ &  Avg. $\mu_{\textrm{all}}$ & Std. $\mu_{\textrm{all}}$ & Avg. RMSD & Std. RMSD & Avg. GDT & Std. GDT & Avg. run-time & Std. run-time  \\
 \hline\hline
reduced database  &  480.53 &  416.66  &  78.22  &  10.96  &  15.31  &  11.24  &  34.08  &  28.39  &  0.68  &  0.59  \\ 
\hline
full database     &  410.73 &  336.63  &  78.25  &  10.23  &  14.78  &  11.18  &  34.95  &  28.16  &  0.86  &  0.63  \\ 
\hline

\end{tabular}}
\vspace{ -0.5cm}
\end{table*}




For our experimental setup, we use the default settings of the latest version of AlphaFold \footnote{\url{https://github.com/deepmind/alphafold}}. This includes the initial multi-sequence alignment (MSA) step, the five-model ensembles predictions, recycling, output confidence ranking, and amber relaxation. For further details about each step, we refer the reader to \cite{AlphaFold2021} and its supplementary information. \textcolor{black}{We include results from using the high-accuracy full database configuration of the initial AlphaFold MSA step along with the less accurate (and faster) reduced database option}. In order to compute the RMSD and GDT, we need to employ an alignment algorithm. In this paper, we use the built-in alignment PyMOL procedure without outlier rejections \cite{pymol}. \textcolor{black}{The parameters of PyMOL alignment are selected using the default settings, which include an outlier rejection cutoff of $2$, a maximum number of outlier rejection cycles of 5, and the use of the structural superposition step }. We note that these outliers only impact the calculations of the RMSD. 

Our adversarial sequences are generated by randomly sampling 20 sequences from the set $\mathcal{V}_{L,H}$ in \eqref{eqn: space} with $H=5$ and $L = 20$. Then, we pick the sequence that returns the maximum value in RMSD structural distance. We use an AMD EPYC 7702 64-Core Processor with 1 TiB of RAM and NVIDIA A100 GPU. We generate adversarial sequences against the COVID-19 protein sequences from the UniProt database considered by AlphaFold in \cite{jumper2020computational}. The original fasta (file extension for protein sequences) sequence files are available online \footnote{\url{https://ftp.uniprot.org/pub/databases/uniprot/pre_release/covid-19.fasta}}. Additionally, we generate adversarial sequences against most of the the UniProt (Universal Protein resource, a central repository of protein data created by combining the Swiss-Prot, TrEMBL and PIR-PSD databases \cite{uniprot2021uniprot}). \textcolor{black}{Our code is provided as supplementary material.}

\subsection{\textcolor{black}{BLOSUM Threshold Experiment}}

In this subsection, we want to investigate how a change in the bound on biological similarity changes the adversarial sequence. In other words, we show the impact of using different BLOSUM thresholds in set $\mathcal{V}_{L,H}$. As such, we randomly select 6 sequences and generate adversarial sequences by configuring the BLOSUM threshold, $L$, to be 20, 30, and 40 (we use strict equalities to ensure the exact BLOSUM distance) and set $H=5$. For each case, we obtain the RMSD after alignment as reported in the fourth column of Table~\ref{table:blosum}. Furthermore, we present the average confidence percentage level of the prediction of the original (adversarial) sequence as reported by AlphaFold and denoted by $\mu_{\textrm{all}}$ ($\mu'_{\textrm{all}}$). Additionally, in the 6th and 8th columns, we report the average confidence values for the residues that are different between the original and adversarial sequences. These are denoted by $\mu_{\textrm{diff}}$ and $\mu'_{\textrm{diff}}$, respectively. We observe that, in general, when the BLOSUM threshold distance increases, the RMSD also increases. This means that biologically increased distance in the input space, in general, causes higher changes in the output predictions of AlphaFold. In terms of the confidence scores, we observe that the change in the overall average confidence between the original and perturbed sequence is not significant. However, in almost all the considered cases, we notice that the prediction confidence of the altered residues has reduced for the adversarial sequence when compared to the ones reported for the original sequence.

\subsection{\textcolor{black}{Confidence Experiment }}

Given a sequence $S$, per residue, AlphaFold generates an estimate of its prediction confidence in the form of a value in $[0,100]$. This value is called the predicted Local Distance Test (pLDDT) and represents the predicted value on the lDDT-C$\alpha$ metric \cite{mariani2013lddt}.


In this subsection, we answer the following question. Does selecting the residues to be changed based on their low (or high) confidence scores impact the resulting RMSD between the original and adversarial structure prediction? Phrased differently, in terms of the RMSD, we illustrate the impact of using the prediction confidence scores of every residue of the predicted structure of the original sequence in determining the location of the residues to be altered in the adversarial sequence generation method presented in the previous section. As such, five, not cherry picked, randomly selected sequences are used. Then, the locations of the 5 residues to be altered are taken based on three categories as follows. Residues are selected with confidence values near the (\textit{i}) minimum confidence score (MIN. category), (\textit{ii}) the average score (AVG. category), and (\textit{iii}) the maximum confidence score (MAX. category). Results are presented in Table~\ref{table:conf}. We observe that, in general, selecting residues with low or high confidence scores is not related to the amount of the induced RMSD at the output. \textcolor{black}{As such, in our method, the locations of the flipped residues are selected independent of the confidence scores.}

\subsection{COVID-19 Case Studies}
\label{sec: quant case studies}


\textcolor{black}{We apply our adversarial approach to 111 publicly available COVID-19 protein sequences as of the time of this writing per the UniProt database using AlphaFold full database configuration. Additionally, in the supplementary material, we provide complete results using the reduced AlphaFold configuration.} The BLOSUM62 distance between the original and adversarial sequences is at most 20, thus they are biologically close to each other \cite{chothia1986relation, sander1991database}. Given the long list of the considered sequences, we describe only the following. SGTA\_HUMAN Small glutamine-rich tetratricopeptide repeat-containing protein alpha (O43765), HLAA\_HUMAN HLA class I histocompatibility antigen, A alpha chain (P04439), STX17\_HUMAN Syntaxin-17 (P56962), AP3A\_SARS ORF3a (P59632), and MYD88\_HUMAN Myeloid differentiation primary response protein MyD88 (Q99836). The cases covered include homo sapiens and severe acute respiratory syndrome coronavirus 2 (2019-nCoV) (SARS-CoV-2) organisms which provide a wide variety of proteins. The considered sequences vary in length as they range from $n=22$ to $n=2511$.\par

Figures~\ref{fig: intro} and \ref{fig: output} show the aligned predicted structures of the proteins described earlier where the original sequence is given in black and the adversarial sequence is given in red. We observe that, independent of the predicted structure of the original sequence, a small change in the input sequence results in significant changes in the output structures. \textcolor{black}{ The resulting structural distances (similarities) measured in \r{A} (percentage) are given in terms of the RMSD (GDT-TS) in the fourth (sixth) column of Table~\ref{table:full db configurations} for the full database configuration. Furthermore, we report the results using GDT-HA in the eighth column. The high similarity between the original and adversarial sequences is observed from the third column. The similarity percentage is calculated as $100(n-D_{\textrm{ham}}(S,S'))/n$, where $D_{\textrm{ham}}(S,S') \leq H = 5$. The complete results of all the considered proteins, including reduced AlphaFold configuration, are provided in the supplementary material. } \par

As observed from the RMSD and GDT results in Table~\ref{table:full db configurations}, small changes in the input sequence corresponding to only five residues cause AlphaFold to predict structures that are highly divergent from the predicted structure of the original sequence. \textcolor{black}{The last column in Table~\ref{table:full db configurations} reports the total execution time (in days) of running the $20$ adversarial sequences that were randomly selected from the set $\mathcal{V}_{L,H}$, which is shown to scale with the sequence length. We only select 20 samples given the long time incurred by AlphaFold to predict the output structure.}

\textcolor{black}{Additionally, in Table~\ref{table:full db configuration confidence results}, we report the average (deviation) prediction confidence results as for all the residues (designated with subscript `all') and for the 5 altered residues (subscript `diff'). The standard deviation is denoted $\sigma$. We observe that, independent of the average prediction confidence, the RMSD between the original and adversarial predicted structures is always high. This is noted for both the full and reduced database configurations of AlphaFold. Moreover, we observe that AlphaFold predicts the adversarial structure with similar confidence values to the original sequence (e.g., see the 4th and 8th columns in both tables). The same observation holds for the entire sequence and for the altered residues (columns 6 and 10).}


\textcolor{black}{In Tables~\ref{table:full db GDT based on confidence regions I} and \ref{table:full db GDT based on confidence regions II} of the supplementary, we break down GDT scores between the structures of the original and perturbed sequences based on the prediction confidence scores of the original sequence. We use the regions (1 to 4) defined by AlphaFold. As observed w.r.t. all regions, GDT scores are, in general, low.}

\textcolor{black}{For the considered dataset, the values presented in Table~\ref{table:overall results} gauge the overall robustness of AlphaFold to adversarial sequences. As indicated in the documentation of AlphaFold, for better accuracy, the full database configuration incurs a higher execution time compared to the reduced database configuration. 
The reported average values of the RMSD and GDT-TS measures are 14.78\r{A} and 37.95\%, respectively. 
In CASP14 (year 2020), AlphaFold achieved a median GDT-TS score of 92.4\%, and 88\% of their predictions fall under RMSD = 4\r{A} \footnote{\url{https://predictioncenter.org/casp14/index.cgi}}. These results are obtained by comparing the predicted and ground truth structures. The CASP14 AlphaFold results underscore the significance of the values reported in Tables~\ref{table:full db configurations} and \ref{table:overall results}, as they show how small changes in the input sequences could damage the predictions (See columns 6 to 9 in Table~\ref{table:overall results}). The key takeaway is that \textit{AlphaFold is generally not robust even when a basic approach is used to generate perturbations of the input protein sequence.}}


\section{Conclusion}

The groundbreaking progress made in recent years on the prediction of protein folding structures promises to enable profound advances in the understanding of diseases, the mapping of the human proteome, and the design of drugs and therapeutics. However, until these predictions are shown to be robust, we argue that the grand challenge of predictive protein folding persists. In this paper, we have presented the first work in this direction by demonstrating that Protein Folding Neural Networks (PFNNs) are often susceptible to adversarial attacks in the form of minor perturbations to the input protein sequence. These perturbations can induce great changes in the predicted protein structure and the resulting lack of robustness precludes the adoption of such PFNNs in safety-critical applications. We have employed standard protein structural distance and similarity to measure the robustness of AlphaFold. While the perturbation methods employed in this paper were basic for the purposes of illustrating the lack of robustness of PFNNs, the \textcolor{black}{results presented} herein can be readily used as a baseline for future work on adversarial attacks on PFNNs and their robustness. 


\bibliography{aaai23.bib}


\newpage
\begin{table*}[t]
\caption{\small{RMSD, GDT-TS, and GDT-HA results using the reduced database AlphaFold configuration with $L=20$ and $H=5$. The average results correspond to $20$ adversarial samples for each protein ID. Part I of II.}}
\label{table:reduced database results}\centering
 \scalebox{0.78}{
\begin{tabular}{||c c c | c c | c c | c c | c ||} 
 \hline
  Protein ID  & $n$ & Similarity (\%) &  RMSD & Avg. RMSD  & GDT-TS (\%) & Avg. GDT-TS (\%) & GDT-HA (\%) & Avg. GDT-HA (\%) &  run-time (days)  \\
 \hline\hline

A0A663  &  38  &  86.8421  &  3.129  &  1.655  &  78.9474  &  95.0  &  57.2368  &  84.5395  &  0.4173  \\ 
MPROTEI  &  193  &  97.4093  &  2.06  &  1.2348  &  83.1606  &  97.0531  &  60.6218  &  89.6308  &  0.7081  \\ 
NSP2NNN  &  546  &  99.0842  &  19.561  &  11.4867  &  4.2125  &  19.9657  &  0.3663  &  7.0696  &  1.6596  \\ 
NSP4NNN  &  494  &  98.9879  &  11.389  &  2.9463  &  6.5789  &  72.7277  &  0.253  &  51.1943  &  1.6327  \\ 
NSP6NNN  &  290  &  98.2759  &  8.921  &  3.5994  &  29.4828  &  79.7716  &  9.3966  &  62.7629  &  0.9444  \\ 
O00327  &  626  &  99.2013  &  10.445  &  5.7079  &  26.3578  &  52.6637  &  9.6645  &  31.4177  &  0.5849  \\ 
O14745  &  358  &  98.6034  &  12.34  &  8.3659  &  7.6117  &  24.7835  &  0.5587  &  8.6627  &  0.4956  \\
\hline
 \end{tabular}}
\end{table*}

\begin{table*}[t]
\caption{\small{RMSD, GDT-TS, and GDT-HA results using the reduced database AlphaFold configuration with $L=20$ and $H=5$. The average results correspond to $20$ adversarial samples for each protein ID. Part II of II.}}
\label{table:reduced database results}\centering
 \scalebox{0.78}{\begin{tabular}{||c c c | c c | c c | c c | c ||} 
 \hline
  Protein ID  & $n$ & Similarity (\%) &  RMSD & Avg. RMSD  & GDT-TS (\%) & Avg. GDT-TS (\%) & GDT-HA (\%) & Avg. GDT-HA (\%) &  run-time (days)  \\
 \hline\hline

 P52292  &  529  &  99.0548  &  9.065  &  4.3164  &  44.9433  &  75.0071  &  23.3459  &  54.5723  &  0.531  \\ 
P52948  &  1817  &  99.7248  &  56.34  &  39.782  &  1.9125  &  8.2327  &  0.0413  &  2.4869  &  2.0914  \\ 
P56962  &  302  &  98.3444  &  25.013  &  19.3739  &  10.3477  &  18.0505  &  0.4139  &  5.1283  &  0.5233  \\ 
P59594  &  1255  &  99.6016  &  4.432  &  2.9699  &  81.1355  &  88.4303  &  59.6016  &  72.8357  &  0.9841  \\ 
P59595  &  422  &  98.8152  &  26.286  &  18.0807  &  4.0284  &  15.0474  &  0.7109  &  3.2524  &  0.4821  \\ 
P59596  &  221  &  97.7376  &  9.173  &  4.2628  &  30.3167  &  71.1369  &  10.0679  &  55.6618  &  0.4292  \\ 
P59632  &  274  &  98.1752  &  18.052  &  14.8883  &  9.854  &  21.3777  &  1.2774  &  7.6505  &  0.452  \\
\hline
 \end{tabular}}
\end{table*}

\begin{table*}[t]
\caption{\small{RMSD, GDT-TS, and GDT-HA results using the full database AlphaFold configuration with $L=20$ and $H=5$. The average results correspond to $20$ adversarial samples for each protein ID. Part I of II. }}
\vspace{ -0.3cm}
\label{table:full db configurations}\centering
 \scalebox{0.78}{\begin{tabular}{||c c c | c c | c c | c c | c ||} 
 \hline
  Seq. ID  & $n$ & Similarity (\%) &  RMSD & Avg. RMSD  & GDT-TS (\%) & Avg. GDT-TS (\%) & GDT-HA (\%) & Avg. GDT-HA (\%) &  run-time (days)  \\
 \hline\hline

A0A663  &  38  &  86.8421  &  2.886  &  1.7076  &  82.8947  &  95.0  &  63.8158  &  85.7566  &  0.4058  \\ 
MPROTEI  &  193  &  97.4093  &  5.562  &  4.8324  &  37.5648  &  45.2785  &  13.4715  &  23.1995  &  0.5288  \\ 
NSP2NNN  &  546  &  99.0842  &  22.333  &  8.7108  &  8.3791  &  33.6332  &  1.3736  &  16.0119  &  0.885  \\ 
NSP4NNN  &  494  &  98.9879  &  16.995  &  14.4184  &  7.2368  &  13.4261  &  1.0121  &  2.9681  &  0.6312  \\ 
NSP6NNN  &  290  &  98.2759  &  5.41  &  3.2315  &  66.4655  &  85.2543  &  42.931  &  67.944  &  0.5433  \\ 
O00327  &  626  &  99.2013  &  18.69  &  5.133  &  10.1438  &  62.7456  &  0.7188  &  41.1621  &  1.8352  \\ 
O14745  &  358  &  98.6034  &  7.395  &  3.6869  &  23.3939  &  57.2521  &  4.8184  &  33.7919  &  0.8124  \\
\hline
\end{tabular}}
\end{table*}

\begin{table*}[t]
\caption{\small{RMSD, GDT-TS, and GDT-HA results using the full database AlphaFold configuration with $L=20$ and $H=5$. The average results correspond to $20$ adversarial samples for each protein ID. This is part II of II. }}
\vspace{ -0.3cm}
\label{table:full db configurations}\centering
 \scalebox{0.78}{\begin{tabular}{||c c c | c c | c c | c c | c ||} 
 \hline
  Seq. ID  & $n$ & Similarity (\%) &  RMSD & Avg. RMSD  & GDT-TS (\%) & Avg. GDT-TS (\%) & GDT-HA (\%) & Avg. GDT-HA (\%) &  run-time (days)  \\
 \hline\hline

P56962  &  302  &  98.3444  &  22.301  &  15.8695  &  12.3344  &  18.6921  &  3.4768  &  5.803  &  0.5959  \\ 
P59596  &  221  &  97.7376  &  7.506  &  2.1014  &  33.9367  &  90.1923  &  11.7647  &  79.2251  &  0.5245  \\ 
P59632  &  274  &  98.1752  &  13.018  &  8.4704  &  24.8175  &  41.0401  &  9.0328  &  21.6834  &  0.5214  \\ 
P59633  &  154  &  96.7532  &  20.993  &  15.7533  &  5.3571  &  15.7468  &  0.3247  &  4.3912  &  0.4625  \\ 
P59634  &  63  &  92.0635  &  2.514  &  1.5476  &  82.1429  &  94.3849  &  58.7302  &  77.619  &  0.4214  \\ 
P59635  &  122  &  95.9016  &  6.527  &  4.9216  &  37.7049  &  58.4016  &  14.3443  &  34.7746  &  0.4458  \\ 
P59636  &  98  &  94.898  &  5.647  &  1.8721  &  47.7041  &  89.3622  &  24.4898  &  78.3673  &  0.4361  \\
\hline
\end{tabular}}
\end{table*}

\begin{table*}[t]
\caption{\small{Prediction confidence results using the reduced database AlphaFold configuration with $L=20$. This is part I of II. }}
\label{table:reduced db configuration confidence results}\centering
 \scalebox{0.78}{\begin{tabular}{||c c | c | c c | c c | c c | c c||} 
 \hline
    Protein ID  & $n$ & RMSD &  $\mu_{\textrm{all}}$ & $\sigma_{\textrm{all}}$  & $\mu_{\textrm{diff}}$ & $\sigma_{\textrm{diff}}$ &  $\mu'_{\textrm{all}}$ &  $\sigma'_{\textrm{all}}$  &  $\mu'_{\textrm{diff}}$ &  $\sigma'_{\textrm{diff}}$  \\
 \hline\hline
A0A663  &  38  &  3.129  &  75.048  &  11.991  &  71.45  &  1.508  &  74.755  &  11.327  &  74.862  &  4.587  \\ 
MPROTEI  &  193  &  2.06  &  89.041  &  7.416  &  87.138  &  2.009  &  87.461  &  8.837  &  73.19  &  3.874  \\ 
NSP2NNN  &  546  &  19.561  &  74.536  &  17.778  &  92.656  &  1.044  &  73.385  &  17.646  &  91.094  &  1.21  \\ 
NSP4NNN  &  494  &  11.389  &  83.936  &  8.896  &  89.444  &  1.356  &  83.913  &  8.92  &  87.668  &  1.288  \\ 
NSP6NNN  &  290  &  8.921  &  77.331  &  11.445  &  73.286  &  2.769  &  76.104  &  12.64  &  65.228  &  0.762  \\ 
O00327  &  626  &  10.445  &  65.242  &  25.756  &  32.584  &  3.532  &  65.128  &  25.804  &  32.296  &  1.221  \\ 
O14745  &  358  &  12.34  &  72.042  &  17.526  &  88.666  &  0.565  &  72.051  &  17.223  &  70.686  &  7.533  \\
\hline

\end{tabular}}
\end{table*}

\begin{table*}[t]
\caption{\small{Prediction confidence results using the reduced database AlphaFold configuration with $L=20$. This is part II of II. }}
\label{table:reduced db configuration confidence results}\centering
 \scalebox{0.78}{\begin{tabular}{||c c | c | c c | c c | c c | c c||} 
 \hline
    Protein ID  & $n$ & RMSD &  $\mu_{\textrm{all}}$ & $\sigma_{\textrm{all}}$  & $\mu_{\textrm{diff}}$ & $\sigma_{\textrm{diff}}$ &  $\mu'_{\textrm{all}}$ &  $\sigma'_{\textrm{all}}$  &  $\mu'_{\textrm{diff}}$ &  $\sigma'_{\textrm{diff}}$  \\
 \hline\hline

P52292  &  529  &  9.065  &  87.35  &  19.082  &  64.18  &  15.207  &  87.229  &  19.347  &  49.878  &  11.916  \\ 
P52948  &  1817  &  56.34  &  55.901  &  24.723  &  34.444  &  2.948  &  55.807  &  24.315  &  35.338  &  1.097  \\ 
P56962  &  302  &  25.013  &  69.898  &  22.843  &  92.422  &  1.212  &  69.777  &  22.207  &  77.418  &  5.058  \\ 
P59594  &  1255  &  4.432  &  80.532  &  15.908  &  91.996  &  0.605  &  80.352  &  16.139  &  90.166  &  0.838  \\ 
P59595  &  422  &  26.286  &  68.249  &  23.318  &  34.632  &  2.009  &  67.744  &  23.489  &  32.506  &  1.776  \\ 
P59596  &  221  &  9.173  &  82.821  &  13.832  &  78.104  &  1.736  &  81.492  &  14.351  &  73.706  &  2.18  \\ 
P59632  &  274  &  18.052  &  48.927  &  18.922  &  42.85  &  3.674  &  47.395  &  18.279  &  42.27  &  4.011  \\
\hline
\end{tabular}}
\end{table*}

\begin{table*}[t]
\caption{\small{Prediction confidence results using the full database AlphaFold configuration. This is part I of II.} }
\vspace{ -0.3cm}
\label{table:full db configuration confidence results}\centering
 \scalebox{0.78}{\begin{tabular}{||c c | c | c c | c c | c c | c c||} 
 \hline
  Seq. ID  & $n$ & RMSD &  $\mu_{\textrm{all}}$ & $\sigma_{\textrm{all}}$  & $\mu_{\textrm{diff}}$ & $\sigma_{\textrm{diff}}$ &  $\mu'_{\textrm{all}}$ &  $\sigma'_{\textrm{all}}$  &  $\mu'_{\textrm{diff}}$ &  $\sigma'_{\textrm{diff}}$  \\
 \hline\hline

A0A663  &  38  &  2.886  &  75.736  &  10.537  &  88.266  &  1.015  &  70.263  &  11.001  &  75.346  &  1.262  \\ 
MPROTEI  &  193  &  5.562  &  87.9  &  7.153  &  92.318  &  1.308  &  87.857  &  6.955  &  88.774  &  1.525  \\ 
NSP2NNN  &  546  &  22.333  &  83.737  &  12.059  &  94.264  &  1.16  &  83.833  &  11.752  &  91.19  &  0.457  \\ 
NSP4NNN  &  494  &  16.995  &  81.696  &  9.823  &  89.742  &  1.546  &  81.481  &  9.997  &  88.32  &  1.986  \\ 
NSP6NNN  &  290  &  5.41  &  78.09  &  12.781  &  79.708  &  2.607  &  77.337  &  11.29  &  81.604  &  1.085  \\ 
O00327  &  626  &  18.69  &  65.332  &  25.637  &  34.454  &  4.124  &  65.298  &  26.183  &  29.684  &  3.145  \\ 
O14745  &  358  &  7.395  &  71.834  &  17.666  &  87.546  &  0.83  &  71.453  &  17.52  &  71.164  &  6.795  \\
\hline

\end{tabular}}
\end{table*}

\begin{table*}[t]
\caption{\small{Prediction confidence results using the full database AlphaFold configuration. This is part II of II.} }
\vspace{ -0.3cm}
\label{table:full db configuration confidence results}\centering
 \scalebox{0.78}{\begin{tabular}{||c c | c | c c | c c | c c | c c||} 
 \hline
  Seq. ID  & $n$ & RMSD &  $\mu_{\textrm{all}}$ & $\sigma_{\textrm{all}}$  & $\mu_{\textrm{diff}}$ & $\sigma_{\textrm{diff}}$ &  $\mu'_{\textrm{all}}$ &  $\sigma'_{\textrm{all}}$  &  $\mu'_{\textrm{diff}}$ &  $\sigma'_{\textrm{diff}}$  \\
 \hline\hline

P56962  &  302  &  22.301  &  69.172  &  23.753  &  96.516  &  0.409  &  69.342  &  23.759  &  96.54  &  0.463  \\ 
P59596  &  221  &  7.506  &  82.061  &  12.733  &  76.868  &  1.775  &  81.175  &  13.742  &  72.064  &  2.043  \\ 
P59632  &  274  &  13.018  &  58.367  &  18.783  &  66.136  &  2.029  &  57.364  &  18.794  &  60.926  &  4.362  \\ 
P59633  &  154  &  20.993  &  47.859  &  10.608  &  38.736  &  4.455  &  46.228  &  8.664  &  48.764  &  0.907  \\ 
P59634  &  63  &  2.514  &  81.124  &  14.464  &  91.134  &  2.227  &  79.701  &  14.365  &  84.632  &  4.937  \\ 
P59635  &  122  &  6.527  &  76.232  &  15.74  &  69.514  &  0.897  &  76.129  &  16.043  &  54.806  &  3.045  \\ 
P59636  &  98  &  5.647  &  80.765  &  16.225  &  91.182  &  5.274  &  77.654  &  15.556  &  79.584  &  6.025  \\
\hline

\end{tabular}}
\end{table*}

\begin{table*}[t]
\caption{\small{GDT-TS results based on AlphaFold predefined confidence regions using the full database AlphaFold configuration. Regions $R_1$ to $R_4$ correspond to confidence scores of above 90\%, 70\% to 90\%, 50\% to 70\%, and below 50\%, respectively. This is part I of II.} }
\vspace{ -0.3cm}
\label{table:full db GDT based on confidence regions I}\centering
 \scalebox{0.78}{\begin{tabular}{||c c | c  c | c c | c c | c c ||} 
 \hline
  Seq. ID  & $n$ & $R_1$ GDT (\%)  &  $R_1$ (\%) & 
  $R_2$ GDT (\%)  &  $R_2$ (\%) &
  $R_3$ GDT (\%)  &  $R_3$ (\%) & 
  $R_4$ GDT (\%)  &  $R_4$ (\%) \\
 \hline\hline

A0A663  &  38  &  N/A  &  0.0  &  89.815  &  71.053  &  65.909  &  28.947  &  N/A  &  0.0  \\ 
MPROTEI  &  193  &  38.918  &  50.259  &  36.648  &  45.596  &  31.25  &  4.145  &  N/A  &  0.0  \\ 
NSP2NNN  &  546  &  8.447  &  40.11  &  9.449  &  46.52  &  3.448  &  10.623  &  8.333  &  2.747  \\ 
NSP4NNN  &  494  &  5.67  &  19.636  &  9.91  &  67.409  &  2.917  &  12.146  &  6.25  &  0.81  \\ 
NSP6NNN  &  290  &  85.227  &  7.586  &  76.279  &  74.138  &  42.143  &  12.069  &  11.111  &  6.207  \\ 
O00327  &  626  &  21.014  &  33.067  &  8.43  &  13.738  &  7.188  &  12.78  &  2.767  &  40.415  \\ 
O14745  &  358  &  32.927  &  11.453  &  26.117  &  50.0  &  22.619  &  23.464  &  8.333  &  15.084  \\
\hline

\end{tabular}}
\end{table*}

\begin{table*}[t]
\caption{\small{GDT-TS results based on AlphaFold predefined confidence regions using the full database AlphaFold configuration. Regions $R_1$ to $R_4$ correspond to confidence scores of above 90\%, 70\% to 90\%, 50\% to 70\%, and below 50\%, respectively. This is part II of II.} }
\vspace{ -0.3cm}
\label{table:full db GDT based on confidence regions II}\centering
 \scalebox{0.78}{\begin{tabular}{||c c | c  c | c c | c c | c c ||} 
 \hline
  Seq. ID  & $n$ & $R_1$ GDT (\%)  &  $R_1$ (\%) & 
  $R_2$ GDT (\%)  &  $R_2$ (\%) &
  $R_3$ GDT (\%)  &  $R_3$ (\%) & 
  $R_4$ GDT (\%)  &  $R_4$ (\%) \\
 \hline\hline

P56962  &  302  &  28.958  &  39.735  &  5.952  &  6.954  &  4.583  &  19.868  &  2.97  &  33.444  \\ 
P59596  &  221  &  37.5  &  26.244  &  36.94  &  60.633  &  20.0  &  6.787  &  5.357  &  6.335  \\ 
P59632  &  274  &  N/A  &  0.0  &  34.184  &  35.766  &  35.87  &  25.182  &  14.72  &  39.051  \\ 
P59633  &  154  &  N/A  &  0.0  &  N/A  &  0.0  &  11.486  &  48.052  &  4.688  &  51.948  \\ 
P59634  &  63  &  91.071  &  44.444  &  76.471  &  26.984  &  84.722  &  28.571  &  N/A  &  0.0  \\ 
P59635  &  122  &  54.762  &  17.213  &  39.286  &  51.639  &  29.167  &  19.672  &  19.643  &  11.475  \\ 
P59636  &  98  &  61.735  &  50.0  &  46.875  &  24.49  &  22.368  &  19.388  &  16.667  &  6.122  \\
\hline

\end{tabular}}
\end{table*}

\end{document}